\documentclass[twoside]{article}

\usepackage[accepted]{aistats2023}
\usepackage{natbib}
\setcitestyle{authoryear,open={(},close={)}}

\usepackage[utf8]{inputenc} 
\usepackage[T1]{fontenc}    
\usepackage[hidelinks]{hyperref}       
\usepackage{url}            
\usepackage{booktabs}       
\usepackage{amsfonts}       
\usepackage{nicefrac}       
\usepackage{microtype}      
\usepackage{xcolor}         
\usepackage{wrapfig}
\usepackage{sidecap}

\usepackage{amsmath}
\usepackage{amssymb}
\usepackage{mathtools}
\usepackage{amsthm}

\usepackage[capitalize,noabbrev]{cleveref}

\usepackage{algorithm}
\usepackage{algorithmic}

\theoremstyle{plain}
\newtheorem{theorem}{Theorem}[section]

\newtheorem{lemma}[theorem]{Lemma}

\theoremstyle{definition}
\newtheorem{definition}[theorem]{Definition}

\theoremstyle{remark}

\DeclareMathOperator*{\argmin}{\arg\!\min}

\DeclareMathOperator{\proj}{proj}
\renewcommand{\vec}[1]{\mathbf{#1}}
\newcommand{\mat}[1]{\mathbf{#1}}

\newcommand{\x}{\vec{x}}

\newcommand{\z}{\vec{z}}

\newcommand{\y}{\vec{y}}

\newcommand{\w}{\vec{w}}
\newcommand{\J}{\mat{J}}

\def\T{^{\intercal}} 
\def\R{\mathbb{R}}  
\newcommand{\M}{\mathcal{M}} 

\newcommand{\inv}{\ensuremath{^{-1}}}
\newcommand{\inner}[2]{\ensuremath{\left\langle #1, #2\right\rangle}}
\newcommand{\norm}[1]{\left\lVert#1\right\rVert}

\DeclareMathOperator*{\Unp}{Unp}
\DeclareMathOperator*{\reach}{reach}
\DeclareMathOperator*{\dual}{dual}

\runningauthor{Hauschultz, Palm, Moreno-Muñoz, Detlefsen, du Plessis, Hauberg}
\begin{document}
\twocolumn[

\aistatstitle{Is an encoder within reach?}

\aistatsauthor{
  Helene Hauschultz\\   Department of Mathematics,\\
  Aarhus University \\
  \texttt{hhauschultz@math.au.dk} 
  \And
  Rasmus Berg Palm \\   \texttt{rasmusbergpalm@gmail.com}
  \And
  Pablo Moreno-Muñoz\\   Technical University of Denmark, \\
  DTU Compute \\
  \texttt{pabmo@dtu.dk}
}
\aistatsauthor{
  Nicki Skafte Detlefsen\\   Technical University of Denmark, \\
  DTU Compute \\
  \texttt{nsde@dtu.dk}
  \And
  Andrew Allan du Plessis\\  Department of Mathematics,\\
  Aarhus University \\
  \texttt{matadp@math.au.dk}
  \And
  Søren Hauberg \\   Technical University of Denmark, \\
  DTU Compute \\
  \texttt{ sohau@dtu.dk}
}
 ]

\begin{abstract}
 The encoder network of an autoencoder is an approximation of the nearest point projection onto the manifold spanned by the decoder. A concern with this approximation is that, while the output of the encoder is always unique, the projection can possibly have infinitely many values. This implies that the latent representations learned by the autoencoder can be misleading. Borrowing from geometric measure theory, we introduce the idea of using the reach of the manifold spanned by the decoder to determine if an optimal encoder exists for a given dataset and decoder. We develop a local generalization of this reach and propose a numerical estimator thereof. We demonstrate that this allows us to determine which observations can be expected to have a unique, and thereby trustworthy, latent representation. As our local reach estimator is differentiable, we investigate its usage as a regularizer and show that this leads to learned manifolds for which projections are more often unique than without regularization.
\end{abstract}

\section{Encoders as projectors}
  A good \emph{learned representation} has many desiderata \citep{bengio2013representation}. The perhaps most elementary constraint placed over most learned representations is that a given observation $\x$ should have a \emph{unique} representation $\z$, at least in distribution. In practice this is ensured by letting the representation be given by the output of a function, $\z = g(\x)$, often represented with a neural network. 

  The \emph{autoencoder} \citep{rumelhart1986learning} is an example where uniqueness of representation is explicitly enforced, even if its basic construction does not suggest unique representations. In the most elementary form, the autoencoder consists of an \emph{encoder} $g_{\psi}: \R^D \rightarrow \R^d$ and a \emph{decoder}
  $f_{\phi}: \R^d \rightarrow \R^D$, parametrized by $\psi$ and $\phi$, respectively.
  These are trained by minimizing the \emph{reconstruction error} of the training
  data $\{ \x_1, \ldots, \x_N \}$,
  \begin{align}
    \psi^*, \phi^* = \argmin_{\psi, \phi} \sum_{n=1}^N \| f_{\phi}(g_{\psi}(\x_n)) - \x_n \|^2.
  \end{align}
  Here $d$ is practically always smaller than $D$, such that the output of the encoder
  is a low-dimensional latent representation of high-dimensional data. The data is assumed
  to lie near a $d$-dimensional manifold $\M$ spanned by the decoder.
  
  \begin{figure}
      \centering
      \includegraphics[width=0.4\textwidth]{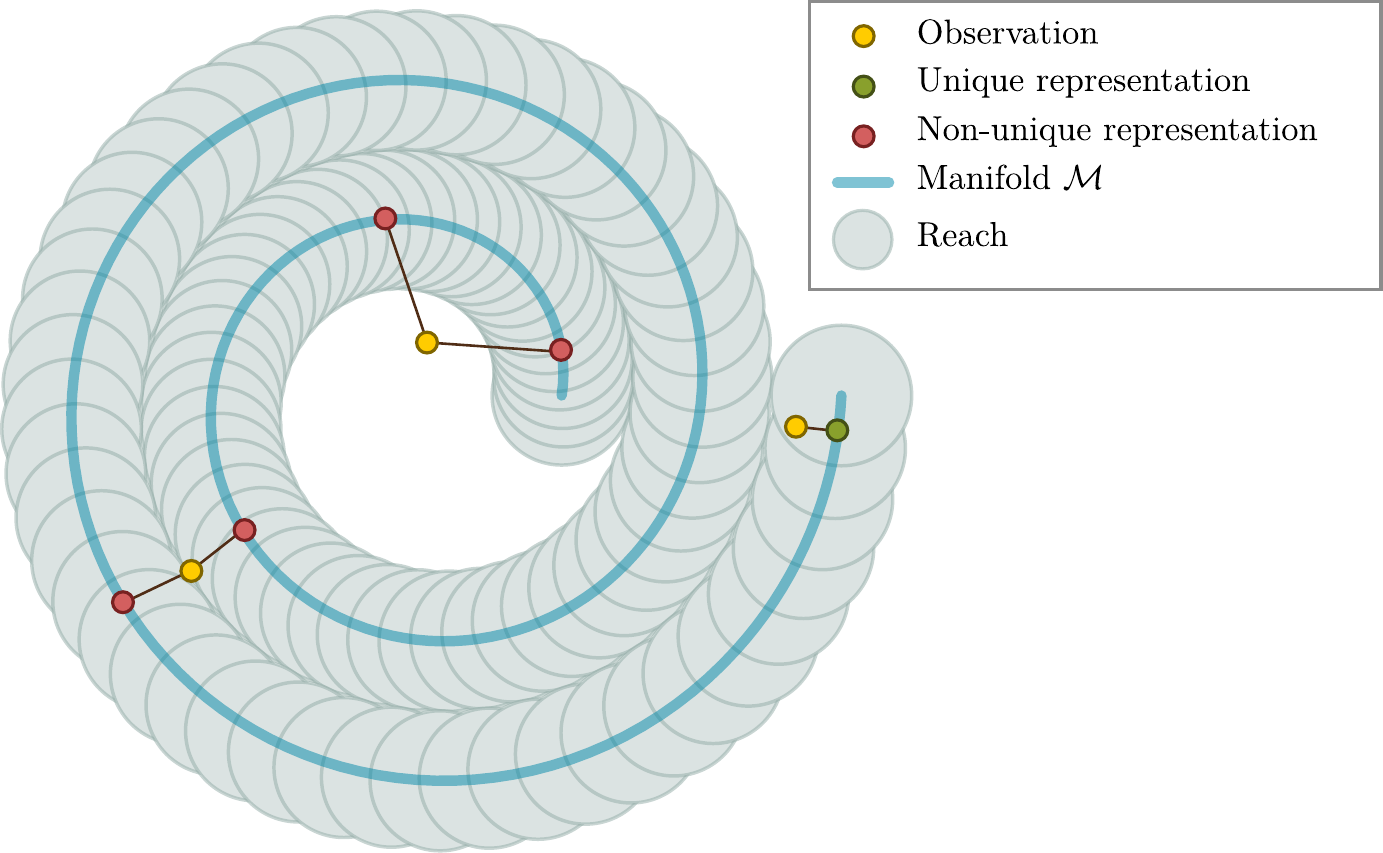}
      \caption{The projection of a point (yellow) onto a nonlinear manifold can take
        unique (green) or multiple values (red) depending on the reach of the manifold. When training data is inside the reach, the encoder can match the projection resulting in more trustworthy representations.}
      \label{fig:encoding_vs_projecting}
  \end{figure}

  For a given decoder, we see that the optimal choice of the encoder is the
  projection onto $\M$, i.e.
  \begin{align}
    g_{\text{optimal}}(\x) &= \proj_{\M}(\x) = \argmin_{\z} \| \x - f(\z) \|^2.
    \label{eq:proj_encoder}
  \end{align}
  For any \emph{nonlinear} choice of decoder $f$, this optimal encoder does \emph{not} exists everywhere. That is, multiple best choices of latent representation may exist for a given point, as the projection is not unique everywhere. As the learned encoder enforces a unique representation, it will choose arbitrarily among the potential representations (see Fig.~\ref{fig:encoding_vs_projecting}). 
  In this case, any analysis of the latent representations can be misleading, as it does not contain the information that another choice of representation would be equally good.
  
  \textbf{But does uniqueness of representations matter?}
  Learned latent representations are used for a variety of tasks, most of which implicitly rely on the representations being unique. The simplest use case of learned representations is \emph{visualization}, i.e.\@ a scatter plot of the latent coordinates. Such plots are often used to form scientific hypotheses about the mechanics of the phenomena that generated the data, e.g.\@ \emph{protein evolution} \citep{riesselman2018deep, detlefsens:proteins:2022}, or \emph{identifying unexplored molecular structures} \citep{sattarov2019novo}. Scatter plots explicitly assume uniqueness of representations (one dot per observation), yet the non-uniqueness of projections \eqref{eq:proj_encoder} suggests that this assumption does not have mathematical backing.
  
  Another common use case is \emph{latent space statistics}. For example, it is common to perform \emph{clustering} of high-dimensional data by finding low-dimensional representations, which are then clustered (either during training or post hoc; see e.g.\@ recent surveys \citep{min2018survey}). This may take the form of $k$-means-style latent clustering \citep{hadipour2022deep}, which assumes that representation averages are well-defined. Another example is \emph{Bayesian optimization} \citep{movckus1975bayesian, stanton2022accelerating} over the latent representations. This assumes the ability to fit stochastic processes to the latent representations. Both of these examples rely on the ability to perform statistical calculations with respect to the learned representations. Unfortunately, practically all statistical calculations rely on the assumption that observation representations are unique. For example, the average of a set of observations with non-unique representations is ill-defined; see e.g.\@ the celebrated work of \citet*{billera2001geometry} for an excellent discussion of this issue.

  The above examples of assuming unique representations are ever-present throughout the literature, yet the mathematical justification is lacking. We investigate methods for ensuring uniqueness, but one could alternatively fully embrace the lack of uniqueness. The latent representation of a single observation would in this case form a set rather than a vector. We do not investigate this direction but note that working with sets is feasible \citep{zaheer2018deep} albeit somewhat more complicated than vectorial representations.
  

  
  \textbf{In this paper} we investigate the \emph{reach} of the manifold $\M$ spanned
  by the decoder $f$. This concept, predominantly studied in geometric measure theory, informs us
  about regions of observation space where the projection onto $\M$ is unique,
  such that trustworthy unique representations exist. If training data resides inside
  this region we may have hope that a suitable encoder can be estimated, leading
  to trustworthy representations. The classic \emph{reach} construction is global in nature,
  so we develop a local generalization that gives a more fine-grained estimate of the
  uniqueness of a specific representation. We provide a new local, numerical, estimator of
  this reach, which allows us to determine which observations can be expected to have unique
  representations, thereby allowing investigations of the latent space to disregard
  observations with non-unique representations. Empirically we find that in large autoencoders, practically all data is outside the reach and risk not having a unique representation. To counter this, we design a reach-based regularizer that penalizes decoders for which unique representations of given data do not exist. Empirically,  this significantly improves the guaranteed uniqueness of representations with only a small penalty in reconstruction error.
  
 \section{Reach and uniqueness of representation}


Our starting question is \emph{which observations $\x$ have a unique representation $\z$ for a given decoder $f$?}  To answer this, we first introduce the \emph{reach} \citep{federer:1959} of the manifold spanned by decoder $f$. This is a \emph{global} scalar that quantifies how far points can deviate from the manifold while having a unique projection. Secondly, we contribute a generalization of this classic geometric construct to characterize the local uniqueness properties of the learned representation.

\subsection{Defining reach}

The nearest point projection $\proj_{\M}$ (Eq.~\ref{eq:proj_encoder}) is a well-defined \emph{function}\footnote{We here stress that a function always returns a single output for a given input.} on all points for which there exists a unique nearest point. We denote this set
\begin{align*}
\Unp(\M) = \{\x\in \R^D : \x\text{ has a unique nearest point in }\M\},
\end{align*}
where $\M = f(\R^d)$ is the manifold spanned by mapping the entire latent space through the decoder. Observations that lie within $\Unp(\M)$ are certain to have a unique optimal representation, but there is no guarantee that the encoder will recover this. With the objective of characterizing the uniqueness of representation, the set $\Unp(\M)$ is a good starting point as here the encoder at least has a chance of finding a representation similar to that of a projection.
However, for an arbitrary manifold $\M$ it is generally not possible to explicitly find the set $\Unp(\M)$. Introduced by \citet{federer:1959}, the \emph{reach} of $\M$ provides us with an implicit way to understand which points are in and outside $\Unp(\M)$.

  \begin{definition}\label{def:reach}
    The \emph{global reach} of a manifold $\M$ is 
    \begin{align}
      \reach(\M) = \inf_{\x\in \M} r_{\text{max}}(\x),
    \end{align}
    where
    \begin{align}
      r_{\text{max}}(\x) = \sup\{r> 0 : B_r(\x)\subset \Unp(\M)\}.
    \end{align}
    Here $B_r$ denotes the open ball of radius $r$.
  \end{definition}
  Hence, $\reach(\M)$ is the greatest radius $r$ such that any open $r$-ball centered on the manifold lies in $\Unp(\M)$. 
  In the existing literature, the \emph{global reach} is referred to as the \emph{reach}; we emphasize the global nature of this quantity as we will later develop local counterparts. 

  Definition~\ref{def:reach} does not immediately lend itself to computation. Fortunately, \citet{federer:1959} provides a step in this direction, through the following result. 
  \begin{theorem}[\citet{federer:1959}]
    Suppose $\M$ is a manifold, then
    \begin{align}\label{eq:fed_reach_calc}
      \reach(\M) = \inf_{\substack{ \x, \y \in \M\\ \y-\x\notin T_\x\M}}   \frac{\norm{\x-\y}^2}{2\norm{P_{N_\x\M}(\y-\x)}},
    \end{align}
    where $P_{N_\x\M}$ is the orthogonal projection onto the normal space of $\M$ at $\x$. If $\y-\x\in T_\x\M$ for all pairs $\x,\y\in \M$ we let $\reach(\M) = \infty$, as $\M$ will be flat and the projection is unique everywhere.
  \end{theorem}

  \begin{figure}
      \centering
      \includegraphics[width=0.4\textwidth]{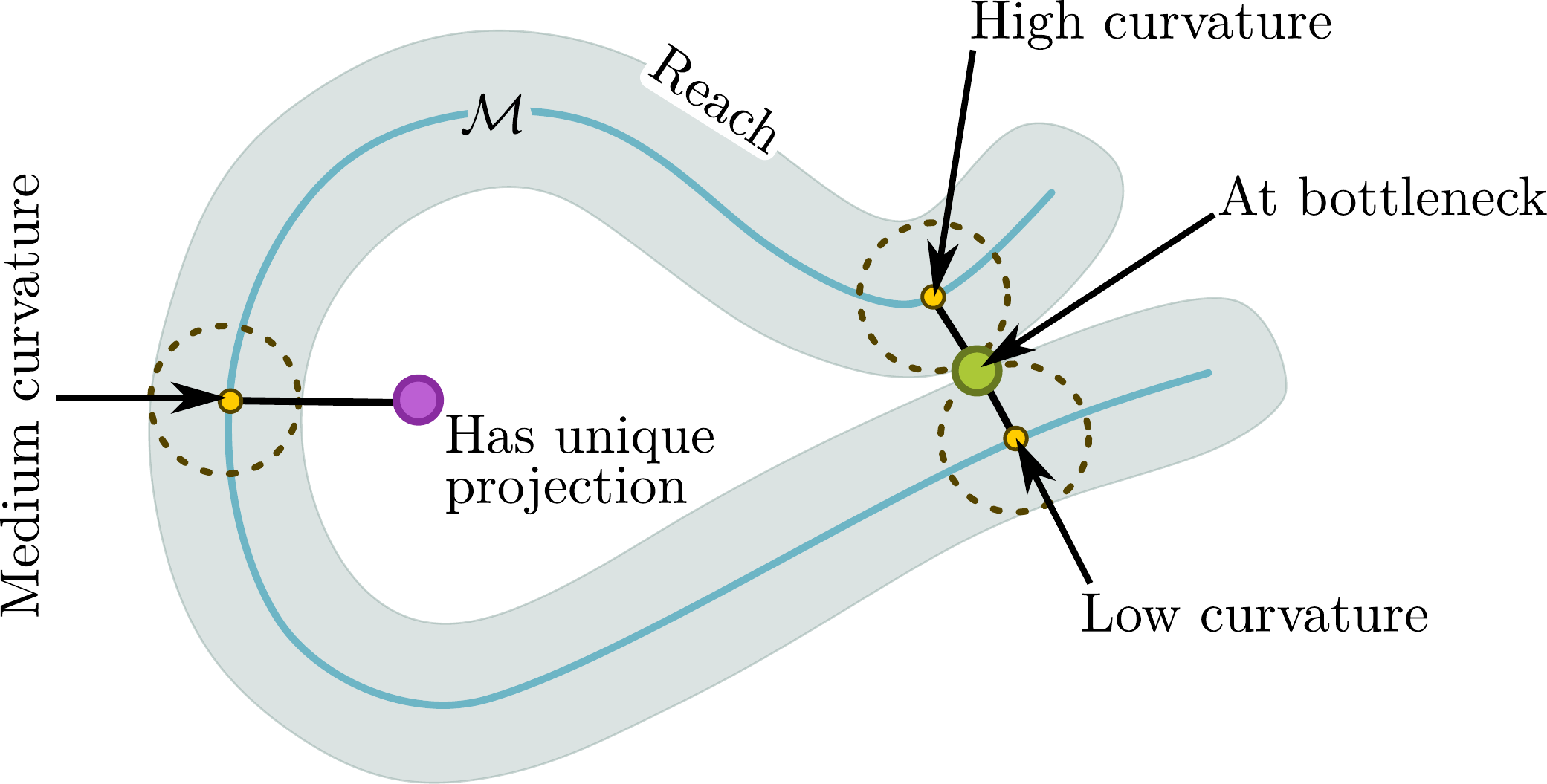}
      \caption{The global reach defines a region around the manifold $\M$ consisting of all points below a certain distance to $\M$. This captures both local manifold curvature as well as global shape.}
      \label{fig:bottleneck}
  \end{figure}
  
  For our objective of understanding which observations have a unique representation, i.e.\@ are inside $\Unp(\M)$, the global reach provides some information. Specifically, the set
  \begin{align}
    \M_r = \left\{ \x | \inf_{\y \in \M} \norm{\y - \x} < \reach(\M) \right\}
  \end{align}
  is a subset of $\Unp(\M)$. This implies that observations $\x$ that are inside $\M_r$ will have a unique projection, such that we can expect the representation to be unique. The downside is that since $\reach(\M)$ is a global quantity, $\M_r$ is an overly restrictive small subset of $\Unp(\M)$.
  Fig.~\ref{fig:bottleneck} illustrates this issue. Note how the global reach in the example is determined by the \emph{bottleneck}\footnote{Not to be confused with \emph{bottleneck network architectures} or the \emph{information bottleneck}.} of the manifold. Even if this bottleneck only influences the uniqueness of projections of a single point, it determines the global reach of the entire manifold. This implies that many points exist outside the reach which nonetheless has a unique projection.
  

  \subsection{Pointwise normal reach}
    \begin{figure}
        \centering
        \includegraphics[width=0.25\textwidth]{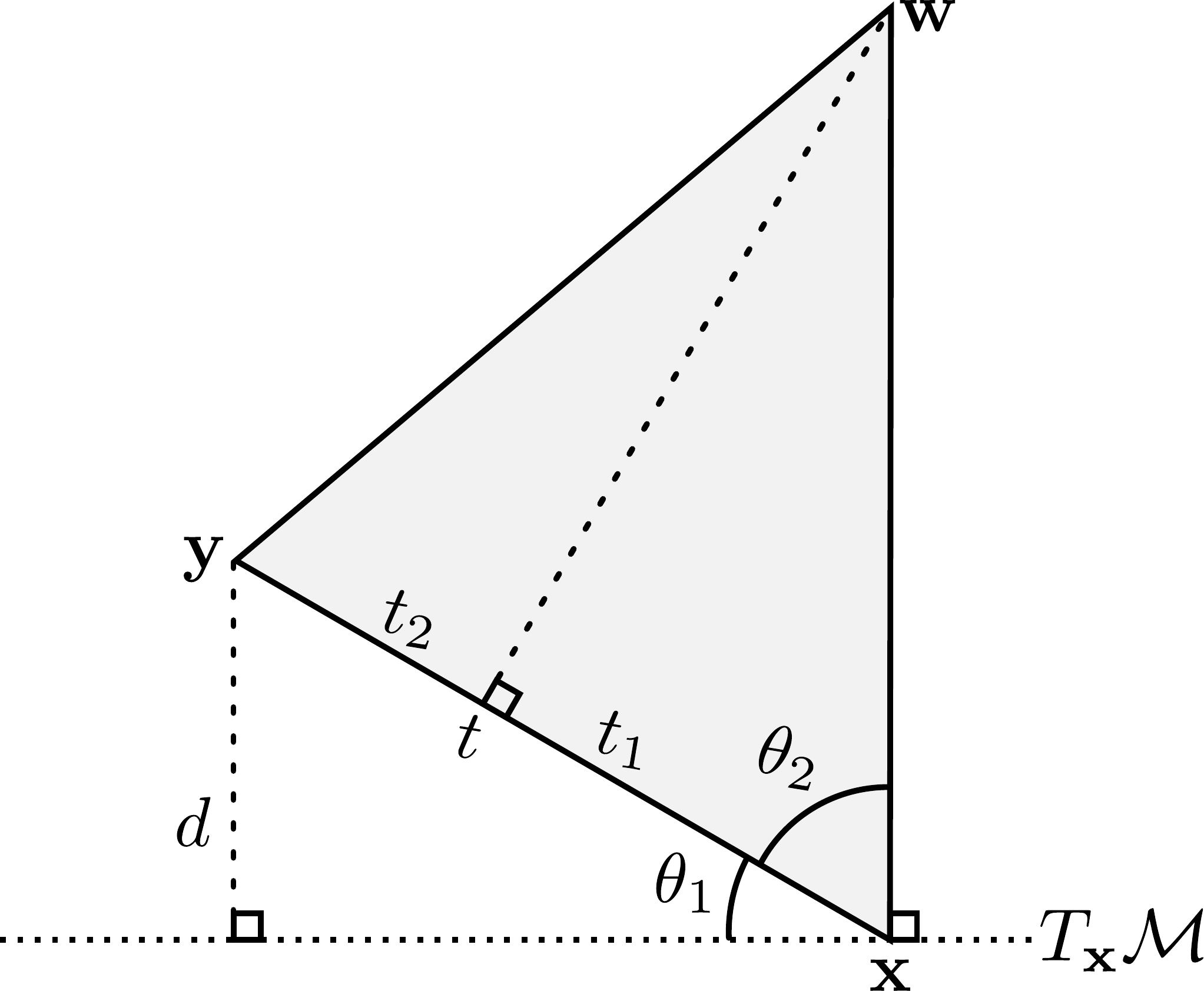}
        \caption{Notation for the proof of theorem~\ref{thm:r_n_in_normal_dir}.}
        \label{fig:proof}
    \end{figure}
    In order to get a more informative notion of reach, we develop a local version, which we, for reasons that will be clear, call the \emph{pointwise normal reach}. For ease of notation denote
    \begin{align}
      R(\x, \y) = \frac{\norm{\x-\y}^2}{2\norm{P_{N_\x\M}(\y-\x)}}
    \end{align}
    for $\x,\y$ with $\y-\x\notin T_\x\M$, else we let $R(\x,\y) = \infty$.
    We then define the pointwise normal reach as the local infimum of eq.~\ref{eq:fed_reach_calc}.
    \begin{definition}[Pointwise normal reach]\label{def:local_normal_reach}
      At a point $\x \in \M$, the pointwise normal reach is
      \begin{align}
        r_N(\x) = \inf_{\y \in M} R(\x,\y).
      \end{align}
    \end{definition}
    In theorem~\ref{thm:r_n_in_normal_dir} below we prove that the local estimate $r_N(\x)$ describes how far we can move along a normal vector at $\x$ and still stay within $\Unp(\M)$. This is useful as we know that $\x$ will lie in the normal space of $\M$ at $\proj_\M(\x)$ (\citet{federer:1959} Thm. 4.8).
    \begin{theorem}\label{thm:r_n_in_normal_dir}
      For all $x\in \M$
      \begin{align}
        B_{r_N(\x)}(\x) \cap N_{\x} \M \subset \Unp(\M),
      \end{align}
      where $N_{\x} \M$ denotes the normal space at $\x$.
    \end{theorem}
    \begin{proof}
      Suppose for the sake of contradiction that there exists $\w\in \left(B_{r_N(\x)}(\x) \cap N_{\x}\M\right)\cap \Unp(\M)^c$. That is, there exists $\y_1,\y_2 \in \M$ such that
      \begin{align}
        d(\w, \M)  = \norm{\y_1-\w} = \norm{\y_2-\w},
      \end{align}
      where $d(\w, \M) = \inf_{\x\in \M} \norm{\x-\w}$.
      In particular, we know there exists $\y\in \M$ such that
      \begin{align}
        \norm{\y-\w}\leq \norm{\x-\w} < r_N(\x).
      \end{align}
      
      Now, let $\theta_1$ denote the (acute) angle between $T_\x\M$ and $\y-\x$, and let $\theta_2$ denote the angle between $\y-\x$ and $\w-\x$. The sum of $\theta_1$ and $\theta_2$ is a right angle, see Fig.~\ref{fig:proof}. Let $t$ be the distance from $\x$ to $\y$. The altitude through the vertex $\w$ divides $\y-\x$ into two line segments. Denote the length of the segment from the foot of the altitude to $\x$, $t_1$, and the length of the segment from the foot to $\y$, $t_2$. Note, $t_2$ will always be less or equal to $t_1$, as $\norm{\y-\w}\leq \norm{\x-\w}$.
      
      By the definition of cosine, $\cos \theta_2 = \frac{t_1}{\norm{\w-\x}} \geq \frac{t/2}{\norm{\w-\x}}$. At the same time $
        \cos\theta_2 = \cos(\pi/2-\theta_1) = \sin \theta_1 = \frac{d}{t},
        $
      where $d = \norm{P_{\w-\x}(\y-\x)}$, and as $\w-\x\in N_\x\M$, $d\leq \norm{P_{N_\x\M}(\y-\x)}$.
      Thus, we have $\frac{t/2}{\norm{\w-\x}} < \frac{d}{t}< \frac{\norm{P_{N_\x\M} (\y-\x)}}{t}$, implying $R(\x,\y)\leq \norm{\w-\x}$, which contradicts $r_N(\x) \leq R(\x,\y)$.
    \end{proof}

    In lemma~\ref{thm:local_reach_bouds} below we show that the pointwise normal reach bounds the reach. For this, we need theorem~4.8(7) from \citet{federer:1959} 
    \begin{lemma}[\citet{federer:1959}]
      Let $\x,\y$ be points on $\M$ with $r_{\text{max}}(\x) > 0$, and let $\vec{n}$ be a normal vector in $N_{\x}\M$, then
      \begin{align}
        \inner{\vec{n}}{\y-\x} \leq \frac{\norm{\y-\x}^2\norm{\vec{n}}}{2r_{\text{max}}(\x)}
      \end{align}
    \end{lemma}

    \begin{lemma}\label{thm:local_reach_bouds}
      For all $\x \in \M$ we have that
      \begin{align}
        \inf_{\y\in B_{2r_N(\x)}(\x)\cap \M} r_N(\y) \leq r_{\text{max}}(\x) \leq r_N(\x).
      \end{align}
    \end{lemma}
    \begin{proof}\phantom{\qedhere} 
      Applying the result from Federer to the vector $\vec{n} = \frac{P_{N_\x\M}(\y-\x)}{\norm{P_{N_\x\M}(\y-\x)}}$ gives
      \begin{align}
        r_{\text{max}}(\x)\leq \frac{\norm{\x-\y}^2\norm{\vec{n}}}{2\norm{\vec{n}}\norm{\x-\y}\cos\theta},
      \end{align}
      where $\theta$ is the angle between $\x-\y$ and $\vec{n}$. Hence, $\cos\theta = \frac{\norm{P_{N_\x\M}(\y-\x)}}{\norm{\y-\x}}$. Thus, for all $\y\in \M$ 
      \begin{align}
        r_{\text{max}}(\x) \leq \frac{\norm{\x-\y}^2}{2\norm{P_{N_\x\M}(\y-\x)}},
      \end{align}
      proving the right inequality.
      Consider $B = B_{r_N(\x)}(\x)$. Suppose there exists $\w \in B$ with $\w \notin \Unp(\M)$. Then $\w \notin N_\x\M$. Hence there exists $\y_1,\y_2\in \M$ such that 
      $
        d(\w, \M) = \norm{\y_1 - \w} = \norm{\y_2 - \w} < r_N(\x).
      $ 
      From \citet{federer:1959} theorem 4.8 we know that $\w\in N_{\y_1}\M, N_{\y_2}\M$. Combining this with lemma \ref{thm:r_n_in_normal_dir} gives that $r_N(\y_1),r_N(\y_2) \leq d(\w,\M)$ and that $d(\x,\M)< \norm{\w-\x}$. We also have that $\norm{\y_1-\x},\norm{\y_2-\x}\leq 2 r_N(\x)$.
      Combining these inequalities gives us that the distance from $\x$ to any point not in $\Unp(\M)$ is greater than $\inf_{\y\in B_{2r_N(\x)}(\x) \cap \M} r_N(\y)$, which implies that
      \begin{equation}
        \inf_{\y \in B_{2r_N(\x)}(\x) \cap \M}r_N(\y) \leq r_{\text{max}}(\x).
        \tag*{\qed}
      \end{equation}
    \end{proof}


We presented the theoretical analysis under the assumption that $\M = f(\R^d)$ is a manifold. Although the theoretical results can be extended to arbitrary subsets of Euclidean space, the experimental setup requires the Jacobian to span the entire tangent space. This might not be the case if $\M$ has self-intersections. The theory can be extended to handle such self-intersections, but this significantly complicates the algorithmic development. See the appendix for a discussion.

  \subsection{Estimating the pointwise normal reach}
    The definition of $r_N$, prompts us to minimize $R(\x,\y)$ over all of $\M$, which is generally infeasible and approximations are in order. As a first step towards an estimator, assume that we are given a finite sample $\mat{S}$ of points on the manifold. We can then replace the infimum in definition~\ref{def:local_normal_reach} with a minimization over the samples. Using that the projection matrix onto $N_{\x}\M$ is given by $P_{N_\x\M} = \mat{I} - \J(\J\T \J)\inv \J\T$, we get the following estimator
    \begin{align}
      \hat r_N(\x) = \min_{\y\in \mat{S}} \frac{\norm{\y-\x}^2}{2\norm{(\mat{I}-\J(\J\T\J)\inv \J\T)(\y-\x)}},
      \label{eq:reach_est}
    \end{align}
    where $\J \in \mathbb{R}^{D \times d}$ is the Jacobian matrix of $f$ at $\x$. Note that since we replace the infimum with a minimization over a finite set, we have that $\hat r_N(\x) \geq r_N(\x)$.

    There are different choices of sampling sets $\mat{S}$. Given a trained autoencoder, a cheap way to obtain samples is to use the reconstructed training data as the sampling set. This will generally be sufficient if the training data is dense on the manifold, but this is rarely the case in high data dimensions.
    The following lemma provides us a way to restrict the area over which we must minimize.
    \begin{lemma}
      For any $\x, \y \in \M$
      \begin{align}
        R(\x,\y) \geq \frac{1}{2} \norm{\x-\y}.
      \end{align}
    \end{lemma}
    \begin{proof}
      Recall that $\y-\x = P_{N_\x\M}(\y-\x) + P_{T_\x\M}(\y-\x)$, as $\R^D = T_\x\M \oplus N_\x\M$. Hence $\norm{\y-\x} \geq \norm{P_{N_\x\M}(\y-\x)}$. The statement, thus, follows from the definition of $R$.
    \end{proof}

  
    \begin{algorithm}[H]
    \caption{Sampling-based reach estimator}\label{alg:sample}
    \begin{algorithmic}
      \STATE radius $\gets r_0$
      \STATE reach $\gets \infty$
      \FOR{$i \gets 1, \ldots, $ num\_batches}
        \STATE samples $\gets$ \texttt{sample\_ball}($\x, \text{radius}, \text{batch\_size}$)
        \STATE projected $\gets$ \texttt{decode}(\texttt{encode}(samples))
        \STATE reach $\gets \min\left(\text{reach}, \texttt{reach\_est}(\x, \text{projected})\right)$
        \STATE radius $\gets 2 \cdot$ reach
      \ENDFOR
    \end{algorithmic}
    \end{algorithm}
    The lemma points towards a simple computational procedure for numerically estimating the pointwise normal reach, which is explicated in algorithm~\ref{alg:sample}. Here $\texttt{reach\_est}$ refers to the application of eq.~\ref{eq:reach_est}.
    The algorithm samples uniformly inside a ball centered on $\x$ and repeatedly shrinks the radius of the ball as tighter estimates of the reach are recovered. We further use the autoencoding reconstruction as an approximation to the projection of $\x$ onto $\M$.

  \subsection{Is a point within reach?}
    Suppose that a point $\x \in \R^D$ is represented by a point on the manifold $f(\z)$. From definition~\ref{def:reach} we know that $\x$ has a unique nearest point on the manifold if
    \begin{align}\label{eq:r_max_boud}
      \norm{\x-f(\z)} < r_{\text{max}}(f(\z)).
    \end{align}
    A point $\x$ which does not satisfy this inequality risks not having a unique nearest point, and hence no unique representation. From lemma~\ref{thm:local_reach_bouds} we know that $r_{\text{max}}(f(\z)) \leq r_N(f(\z))$. So $\x$ risks not having a unique nearest point if
    \begin{align}
      \norm{\x-f(\z)} \geq r_N(f(\z)) \geq r_{\text{max}}(f(\z)).
    \end{align}
    We note that to show that $\norm{\x-f(\z)} \geq r_N(f(\z))$, it is enough to compute
    \begin{align}
      \hat r_N(f(\z)) = \inf_{\substack{\y\in \M\cap B_{2\norm{\x-f(\z)}}(f(\z)) \\ \y \neq f(\z)}} R(f(\z), \y),
    \end{align}
    i.e.\@ limit the search to a ball of radius $2\norm{\x-f(\z)}$. Thus, when we only need to determine if a point is inside the pointwise normal reach, we can pick $r_0 = 2\norm{\x-f(\z)}$ in Algorithm~\ref{alg:sample}.
    
    Notice that given any set of points on the manifold, the resulting estimation of $r_N$ will always be larger than the true value. It means that any point which lies outside the estimated normal reach, will in fact lie outside the true normal reach. However, a point which lies inside the estimated normal reach, risks lying outside the true normal reach, and thus not having a unique projection. 
    
  \subsection{Regularizing for reach}
    The autoencoder minimizes an $l_2$ error which is directly comparable to the pointwise normal reach. This suggests a regularizer that penalizes if the $l_2$ error is larger than the pointwise normal reach. In practice, we propose to use
    \begin{align}
      \mathcal{R}(\x) &= \texttt{Softplus}\left( \norm{f\left(g(\x)\right) - \x} - \hat r_N\left( f\left(g(\x)\right)\right) \right).
    \end{align}
    The reach-regularized decoder then minimizes
    \begin{align}
      \mathcal{L} &= \sum_{n=1}^N \| f(g(\x_n)) - \x_n \|^2 + \lambda \sum_{n=1}^N \mathcal{R}(\x_n),
    \end{align}
    while we do not regularize the encoder.
    We also experimented with a \texttt{ReLU} activation instead of \texttt{Softplus}, but found the latter to yield more stable training. When estimating the pointwise normal reach, $\hat r_N$, we apply Algorithm~\ref{alg:sample} with an initial radius of $r_0 = 2 \| f(g(\x_n)) - \x_n \|$.

\section{Experiments}
Having established a theory and algorithm for determining when a representation can be expected to be unique, we next investigate its use empirically. We first compute  the pointwise local reach across a selection of models to see if it provides useful information. We then carry on to investigate the use of reach regularization. \footnote{The code is available at \url{https://github.com/HeleneHauschultz/is_an_encoder_within_reach}.}

\begin{figure} 
    \includegraphics[width=\columnwidth]{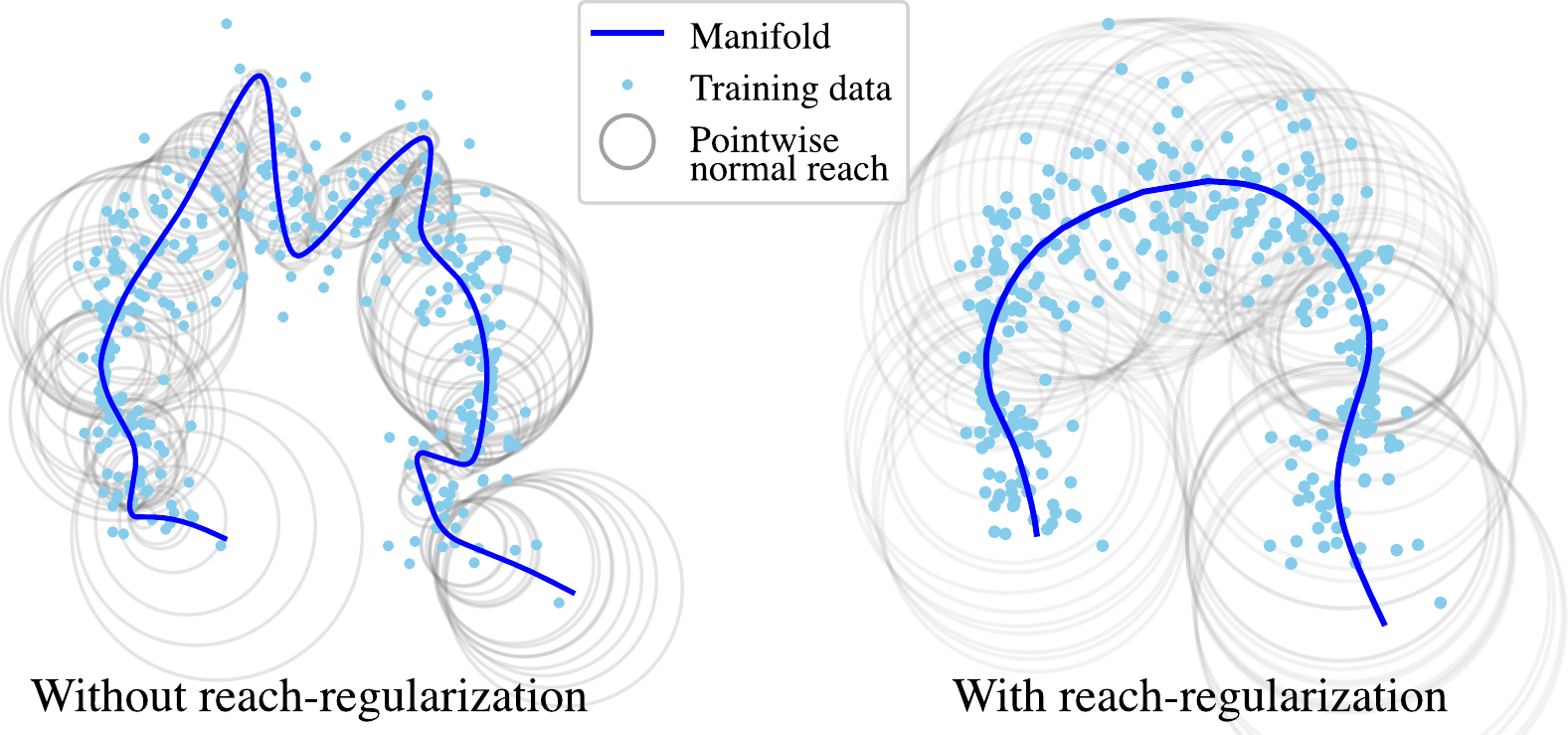}
    \caption{\emph{Left:} An autoencoder trained on noisy points scattered along a circular arc. \emph{Right:} The manifold spanned by the decoder of an autoencoder trained with reach regularization. In both panels, the gray circles illustrates the estimated pointwise normal reach at points along the autoencoder curve.}
    \label{fig:no_reg_reach}
\end{figure}

\subsection{Analysing reach}

\subsubsection{Toy circle}\label{subsec:circle}
We start our investigations with a simple toy example to get an intuitive understanding. We generate observations along a circular arc with added Gaussian noise of varying magnitude. Specifically, we generate approximately $400$ points as $z \mapsto t\left(\sin(z), -\cos(z)\right) + 1.5\cos(z)\epsilon$, where $\epsilon \sim \mathcal{N}(0,1)$. On this, we train an autoencoder with a one-dimensional latent space. The encoder and decoder both consist of linear layers, with three hidden layers with $128$ nodes and with ELU non-linearities.

Figure~\ref{fig:no_reg_reach}(left) shows the data alongside the estimated manifold and its pointwise normal reach.
We observe that the manifold spanned by the decoder has areas with small reach, where the manifold curves to fit the noisy data. The pointwise normal reach seems to well-reflect the curvature of the estimated manifold. The plot illustrates how some of the points end up further away from the manifold than the reach. For some of the points, this is not a problem, as they still have a unique projection onto the manifold. However, some of the points are equally close to different points on the manifold, such that their representation cannot be trusted.


\begin{figure*}[t]
  \centering
  \includegraphics[width=\textwidth]{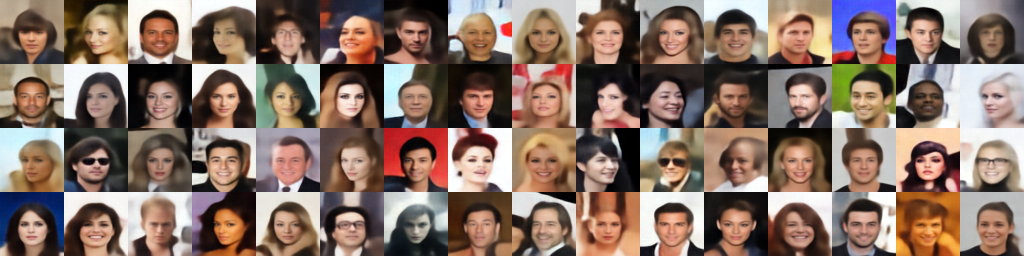}
  \caption{CelebA validation set reconstructions.}
  \label{fig:celeba-val-recon}
\end{figure*}

\subsubsection{CelebA}
To investigate the reach on a non-toy dataset, we train a deep autoencoder on the CelebA face dataset \citep{celeba}. The dataset consists of approximately $200\,000$ images of celebrity faces. 

We train a symmetric encoder-decoder pair that maps the $64 \times 64 \times 3$ images to a $128$ dimensional latent space, and back. The encoder consists of a single 2d convolution operation without stride followed by six convolution operations with stride 2, resulting in a $1 \times 1 \times C$ image. We use $C=128$ channels for all convolutional operations, a filter size of 5 and Exponential Linear Unit (ELU) non-linearities. The decoder is symmetric, using transpose convolutions with stride 2 to upsample and ending with a convolution operation mapping to $64 \times 64 \times 3$. The model is trained for 1M gradient updates on the mean square error loss, with a batch size of 128, using the Adam optimizer with a learning rate of $10^{-4}$. 
Example reconstructions on the validation set are provided in Fig.~\ref{fig:celeba-val-recon}.

After training we estimate the reach of the validation set using the sampling based approach (Alg.~\ref{alg:sample}). Fig.~\ref{fig:plots}(left) plots the reconstruction error $\norm{\x - f(\z)}$ versus the pointwise normal reach. We observe that almost all observations lie outside the pointwise normal reach, implying that we cannot guarantee a unique representation. This is a warning sign that our representations need not be trustworthy.

\begin{figure*}
  \includegraphics[width=\textwidth]{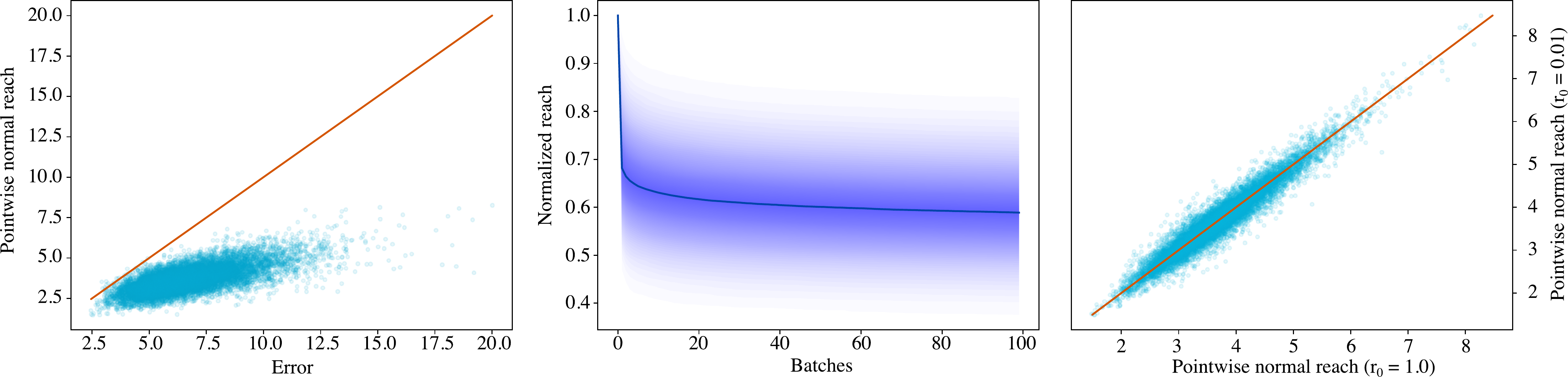}
  \caption{
    \emph{Left:} Estimated reach for CelebA validation samples plotted against the L2 error. Samples below the diagonal red line does not have a unique encoder.
    \emph{Center:} Normalized reach as a function of batches used to estimate the reach. 
    The normalized reach is the estimated pointwise normal reach divided by the estimated pointwise normal reach after the first batch.
    The hyperball sampling reach estimator quickly converges.
    \emph{Right:} Sensitivity analysis of the hyperball sampling reach estimator to the initial hyperball radius. The reach of CelebA validation samples are estimated with initial radii $r_0=1.0$ and $r_0=0.01$ respectively and their final reach after 100 batches are plotted against each other.
  }
  \label{fig:plots}
\end{figure*}

Next we analyze the empirical convergence properties of our estimator on the CelebA autoencoder. Fig.~\ref{fig:plots}(center) shows the average pointwise normal reach over the validation set as a function of the number of iterations in the sampling based estimator. We observe that the estimator converges after just a few iterations, suggesting that the estimator is practical.

The estimator relies on an initial radius for its search. Fig.~\ref{fig:plots}(right) shows the estimated pointwise normal reach on the validation set, plotted for two different initial radii. We observe that the estimator converges to approximately the same value in both cases, suggesting that the method is not sensitive to this initial radius. However, initializing with a tight radius will allow for faster convergence.

\subsection{Reach regularization}
Having established that the pointwise normal reach provides a meaningful measure of uniqueness, we carry on to regularize accordingly.

\subsubsection{Toy circle}
Returning to the example from section~\ref{subsec:circle}, we train an autoencoder of the same architecture with the reach regularization. We pretrain the network $100$ epochs without regularization, and then $2000$ iterations with reach regularization. \looseness=-1

Fig.~\ref{fig:no_reg_reach} (right) shows that reach regularization gives a significantly smoother manifold than without regularization (left panel). The gray circles on the plot indicate that almost all the points are now within the pointwise normal reach, and arguably the associated representations are now more trustworthy.\looseness=-1

\subsubsection{MNIST}\label{section:reach_reg:Mnist}
Next we train an autoencoder on 5000 randomly chosen images from the classes 2, 4 and 8 from MNIST \citep{lecun1998gradient}. We use a symmetric architecture reducing to two dimensional representation through a sequence of $784\rightarrow500\rightarrow250\rightarrow150\rightarrow100\rightarrow50\rightarrow10$ linear layers with ELU non-linearities. We pretrain 5000 epochs without any regularization, and proceed with reach regularization enabled. Figure~\ref{fig:mnist_reg} (left) shows the percentage of points which lies within reach of the estimated manifold. We observe that reach regularization slightly increases the reconstruction error (see example reconstructions in fig.~\ref{fig:mnist_recon}), as any regularization would, while significantly increasing the percentage of points that are known to have a unique representation. This suggests that reach regularization only minimally changes reconstructions while giving a significantly more smooth model, which is more reliable.

Figure~\ref{fig:mnist_reg} (center) shows the latent representations given by the pretrained autoencoder without regularization, while fig.~\ref{fig:mnist_reg} (right) shows the latent representations after an additional 200 epochs with reach regularization. The latent representations with corresponding data points outside reach, that is, where the reconstruction error is greater than the pointwise normal reach at the reconstructed point is plotted in red. The points inside reach are plotted in green. We observe that after regularization, significantly more points can be expected to be unique and thereby trustworthy. Note that the latent configuration is only changed slightly after reach regularization, which suggests that the expressive power of the model is largely unaffected by the reach regularization.


\begin{figure*}
    \includegraphics[width=\textwidth]{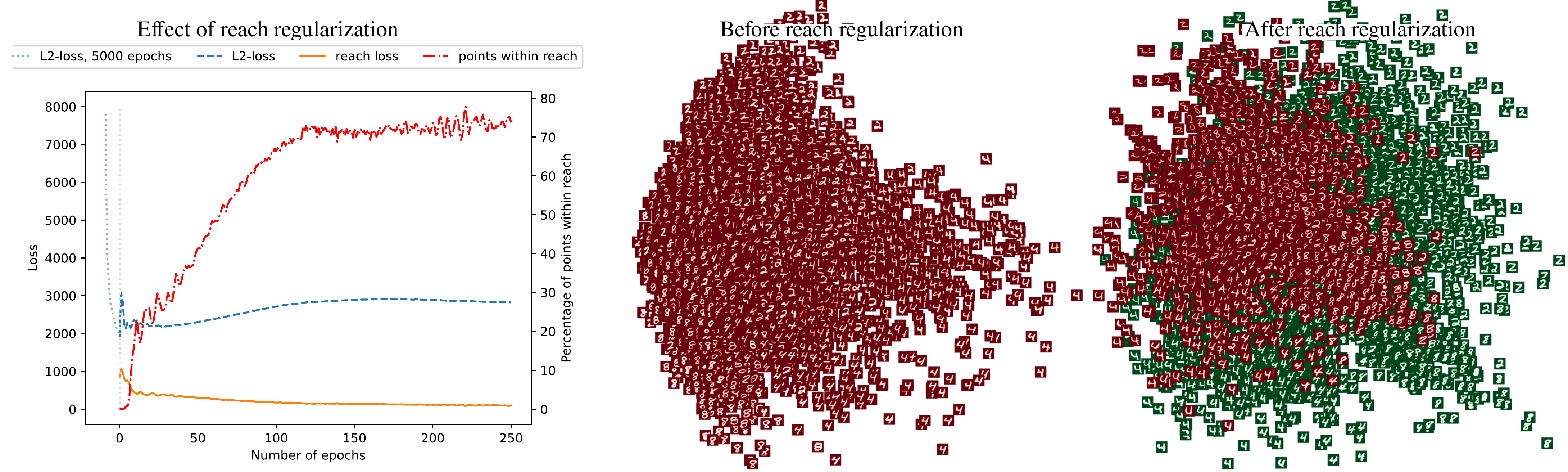}
    \caption{The effect of reach regularization on an MNIST model.
      \emph{Left:} The plot shows that the percentage of points within reach increases, while the $l_2$-loss is nearly unchanged. 
      We plot the loss curve from the initial 5000 epochs without regularization, to show how the $l_2$-loss behaves when regularizing.
      \emph{Center \& right:} Latent representations of the MNIST autoencoder before and after the reach regularization (visualized using PCA). The red numbers are outside the reach, while the green are within. Reach regularization smoothens the decoder to increase reach with minimal changes to both reconstructions and latent configuration.
    }
    \label{fig:mnist_reg}
\end{figure*}

\begin{figure}
  \includegraphics[width=\linewidth]{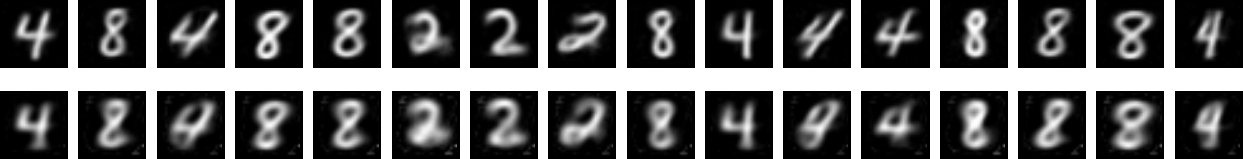}
  \caption{Reconstructions of MNIST images before (top) and after (bottom) reach regularization. The reach regularization only minimally reduces reconstruction quality, while significantly improving upon representation uniqueness.}
  \label{fig:mnist_recon}
\end{figure}

\section{Related work}
  Representation learning is a foundational aspect of current machine learning, and the discussion paper by \citet{bengio2013representation} is an excellent starting point. As is common, \citet{bengio2013representation} defines a representation as the output of a function applied to an observation, implying that a representation is unique. In the specific context of autoencoders, we question this implicit assumption of uniqueness as many equally good representations may exist for a given observation. While only studied here for autoencoders, the issue applies more generally when representations span submanifolds of the observation space.
  
  In principle, probabilistic models may place multimodal distributions over the representation of an observation in order to reflect lack of uniqueness. In practice, this rarely happens. For example, the highly influential \emph{variational autoencoder} \citep{kingma:iclr:2014, rezende:icml:2014} amortizes the estimation of $p(\z|\x)$ such that it is parametrized by the output of a function. Alternatives relying on Monte Carlo estimates of $p(\z|\x)$ do allow for capturing non-uniqueness \citep{pmlr-v70-hoffman17a}, but this is rarely done in practical implementations. That Monte Carlo estimates provide state-of-the-art performance is perhaps indicative that coping with non-unique representations is important. Our approach, instead, aim to determine which observations can be expected to have a unique representation, which is arguably simpler than actually finding the multiple representations.
  
  Our approach relies on the reach of the manifold spanned by the decoder. This quantity is traditionally studied in geometric measure theory as the reach is informative of many properties of a given manifold. For example, manifolds which satisfy that $\reach(\M)> 0$ are $C^{1,1}$, i.e.\@ the transition functions are differentiable with Lipschitz continuous derivatives. In machine learning, the reach is, however, a rarely used concept. \citet{Fefferman2016} investigates if a manifold of a given reach can be fitted to observed data, and develops the associated statistical test. Further notable exceptions are the multichart autoencoder by \citet{schonsheck2020chart}, and the adaptive clustering of \citet{besold2020adaptive}. Both works rely on the reach as a tool of derivation. Similarly, \citet{chae2021likelihood} relies on the assumption of positive reach when deriving properties of deep generative models. These works all rely on the global reach, while we have introduced a local generalization. 

  The work closest to ours appears to be that of \citet{aamari2019estimating} which studies the convergence of an estimator of the global reach \eqref{eq:fed_reach_calc}. This only provides limited insights into the uniqueness of a representation as the global reach only carries limited information about the local properties of the studied manifold. We therefore introduced the pointwise normal reach alongside an estimator thereof. This gives more precise information about which observations can be expected to have a unique representation.

\section{Discussion}
  The overarching question driving this paper is \emph{when can representations be expected to be unique?} 
  Though commonly assumed, there is little mathematical reason to believe that the choice of optimal representation is generally unique. The theoretical implication of this is that enforcing uniqueness on non-unique representations leads to untrustworthy representations. 
  
  We provide a partial answer for the question in the context of autoencoders, through the introduction of the \emph{pointwise normal reach}. This provides an upper bound for a radius centered around each point on the manifold spanned by the decoder, such that any observation within the ball has a unique representation. This bound can be directly compared to the reconstruction error of the autoencoding to determine if a given observation might not have a unique representation. This is a step towards a systematic quantification of the reliability and trustworthiness of learned representations.
  
  Empirically, we generally find that most trained models do not ensure that representations are unique. For example, on CelebA we found that almost no observations were within reach, suggesting that uniqueness was far from ensured. This is indicative that the problem of uniqueness is not purely an academic question, but one of practical importance.
  
  We provide a sampling estimator of the pointwise normal reach, which is guaranteed to upper bound the true pointwise normal reach. The estimator is easy to implement, with the main difficulty being the need to access the Jacobian of the decoder. This is readily accessible using forward-mode automatic differentiation, but it can be memory-demanding for large models.
  
  It is easy to see that the sample-based pointwise normal reach estimator converges to the correct value in the limit of infinitely many samples. We, however, have no results on the rate of convergence. In practice we observe that the estimator converges in a few iterations for most models, suggesting the convergence is relatively fast.  In practice, the estimator, however, remains computationally expensive.

  While we can estimate the pointwise normal reach quite reliably even for large models within manageable time, the estimator is currently too expensive to use for regularization of large models. On small models, we observe significant improvements in the uniqueness properties of the representations at minimal cost in terms of reconstruction error. This is a promising result and indicative that it may be well worth using this form of regularization. While more work is needed to speed up the estimating of pointwise normal reach, our work does pave a path to follow.

\subsubsection*{Acknowledgements}
This work was supported by research grants (15334, 42062) from VILLUM FONDEN. This project has also received funding from the European Research Council (ERC) under the European Union's Horizon 2020 research and innovation programme (grant agreement 757360). This work was funded in part by the Novo Nordisk Foundation through the Center for Basic Machine Learning Research in Life Science (NNF20OC0062606).  Helene Hauschultz is partly financed by Aarhus University Centre for Digitalisation, Big Data and Data Analytics (DIGIT).


\bibliography{paper}

\newpage
\appendix

\section{Appendix}

\subsection{Extending the pointwise normal reach to the non-manifold setting}
\citet{federer:1959} introduces reach for arbitrary subsets of Euclidean space.  In this situation $T_\x\M$ and $N_\x\M$ denote the tangent- and normal cone. 

\begin{definition}
Let $\M \subset \R^D$ denote an arbitrary subset and let $x\in \M$. Then $v\in \R^D$ is a tangent vector for $\M$ at $\x$ if either $v= 0$ or if for every $\varepsilon> 0 $ exists $\y\in \M$ with
\begin{equation}
    0 <\norm{\y-\x} < \varepsilon \quad \text{and}\quad \norm{\frac{\y-\x}{\norm{\y-\x}} - \frac{v}{\norm{v}}} < \varepsilon.
\end{equation}
Let $T_\x\M$ denote the set of tangent vectors for $\M$ at $\x$. A vector $w\in \R^D$ is a normal vector for $\M$ at $\x$ if 
\begin{equation}
    \inner{w}{v} \leq 0 \text{ for all } v\in T_\x\M.
\end{equation}
Let $N_\x\M$ denote the set of all normal vectors for $\M$ at $\x$.
\end{definition}

We can extend theorem~\ref{thm:r_n_in_normal_dir} and Lemma~\ref{thm:local_reach_bouds} to the general situation as defined by Federer. To extend Theorem~\ref{thm:r_n_in_normal_dir} it is sufficient to prove that for any $v\in N_\x\M$ and $u\in \R^D$, $\norm{P_v(u)}\leq d(u,T_\x\M)$.

\begin{lemma}
For any $v\in N_\x\M$ and $u\in \R^D$ with $\inner{v}{u} \geq 0$, $\norm{P_v(u)}\leq d(u,T_\x\M)$.
\end{lemma}
\begin{proof}
For a subset $A\subset \R^D$, $\dual(A) = \{v\in \R^D: \inner{a}{v} \leq 0 \text{ for all }a\in A\}$. First we prove that $d(u,\dual(v)) = \norm{P_v(u)}$. Note that we can write $u = u_v + u_{v^\bot}$, where $u_v = P_v(u)$ and $u_{v^\bot} \in v^\bot$. Then
\begin{equation}
    \begin{split}
        d(u,\dual(v)) =& \inf_{w\in\dual(v)} \norm{u-w} = \inf_{w\in \dual(v)}\norm{u_v + u_{v^\bot} - w}\\ =& \inf_{w\in \dual(v)} \norm{u_v} + \norm{u_{v^\bot} -w} -2\inner{u_v}{ w}.
    \end{split}
\end{equation}
As $\inner{u_v}{w} \leq 0$, it follows that the infimum is achieved when $w = u_{v^\bot}$. By the definition of the dual it follows that $\dual(v) \supset \dual(N_\x\M) \supset T_\x\M$. Hence 
$d(u, T_\x\M) = \inf_{w\in T_\x\M} \norm{u-w} \geq \inf_{w\in \dual(v)}\norm{u-w} = \norm{P_u(v)}$.
\end{proof}

To extend lemma~\ref{thm:local_reach_bouds} note that if $r_{max}(x) > 0$, then $T_\x\M$ is convex \citep[Thm 4.8 (12)]{federer:1959}. Let $\y\in \M$. If $\y-\x\in T_\x\M$ then $R(\x,\y) = \infty$. Otherwise, as $T_\x\M$ is a convex cone, there exists $n\in N_\x\M$ such that $\inner{n}{\y-\x} \geq 0$. In that case $\inner{n}{\x-\y} = \norm{P_n(\y-\x)}$, so applying Lemma 2.5 gives the result.

Though the theory can be extended to general subspaces, the manifold assumption is important for the experimental setup.
An important assumption for the estimator (\ref{eq:reach_est}) is that the Jacobian spans the entire tangent space. 
If this is not the case, this estimator does not estimate the pointwise normal reach. The reason being, the length of the projection onto the orthogonal complement of the Jacobian is not necessarily the distance to the tangent space. 
It is clear that when we want to study the uniqueness of latent representations, if the decoder is not injective, it automatically has areas without unique representations. So if the decoder is not injective, we should already be wary about trusting the latent representations. 

\subsection{Reach estimation in increasing ambient dimension}
In the following experiment we want to see the behavior of the reach estimator when the dimension in which the manifold is embedded increases. 
We consider the graph $(\x,\y)\mapsto U_n(\x,\y,\x^2 + \y^2, 0,\dots , 0)$, where $U_n\in O(n)$ is an orthogonal matrix. That is, we embed the quadratic surface $(\x,\y,\x^2 + \y^2)$ isometrically into $\R^n$. We then estimate the pointwise normal reach in $\vec{0}$ with one iteration of Algorithm \ref{alg:sample} with an initial radius of $5$ and a sample size of $10$. We estimate the pointwise normal reach 100 times in each dimension and take the average of these. The true value of the pointwise normal reach is $r_N(\vec{0}) = 0.5$.
\begin{figure}\label{fig:quadratic_reach}
    \centering
    \includegraphics[width =0.99\linewidth]{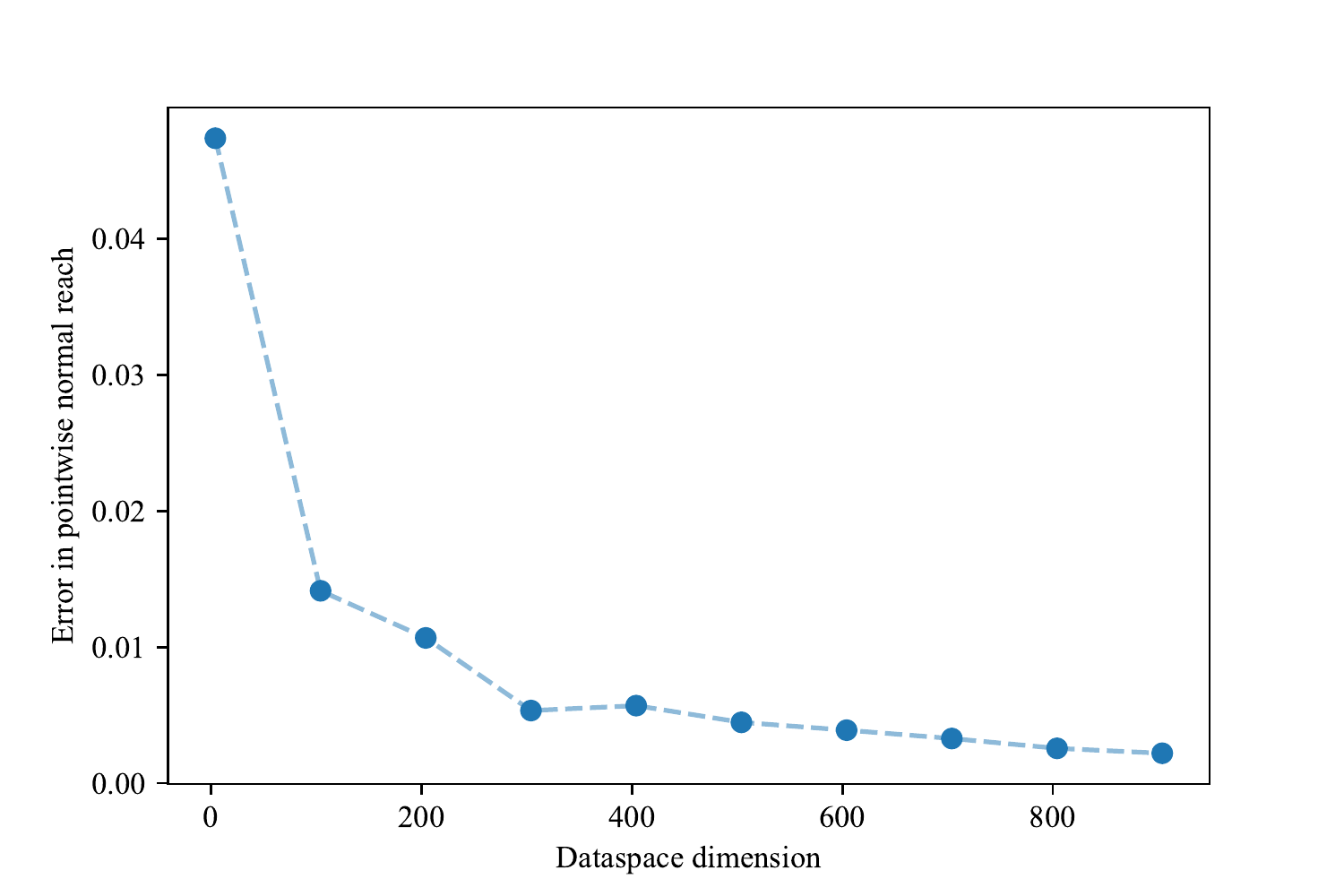}
    \caption{The plot shows how the average overestimation of the pointwise normal reach of an 2-dimensional quadratic surfaces isometrically embedded into a higher dimensional space goes down as the ambient dimension goes up.}
\end{figure}
Figure \ref{fig:quadratic_reach} shows how the average overestimation of the pointwise normal reach goes down as the ambient dimension goes up.

\subsection{Reconstruction error in test set during reach regularization}
We extend the experiment from Section~\ref{section:reach_reg:Mnist} where we perform the reach regularization on an autoencoder trained on a subset of the MNIST data. At each iteration we compute the reconstruction error of a test set. We see that the test error is similar to the training error, suggesting that the model generalizes well to the data. This implies that a model having data points outside reach does not determine that the model does not generalize well to the data. Furthermore, reach regularization does necessarily impact the generalization of the model. 
\begin{figure}
    \centering
    \includegraphics[width = 0.99\linewidth]{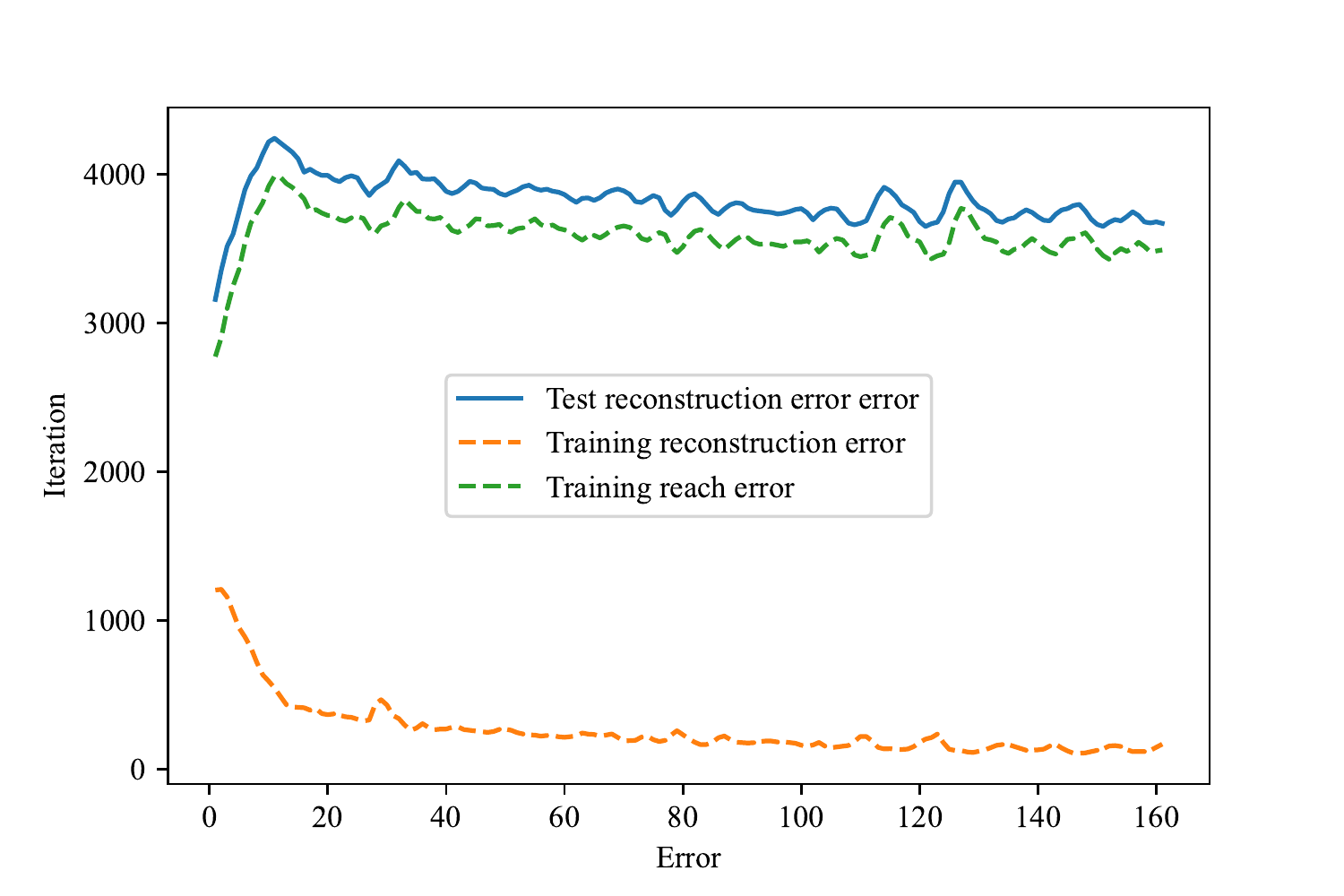}
    \caption{The plot shows 160 iterations of reach regularization of the autoencoder trained on the MNIST dataset, as in Section~\ref{section:reach_reg:Mnist}. The blue line shows the average reconstruction error on the test set, the green line shows the average reconstruction error on the training set, and the orange line shows the reach loss.}
\end{figure}

\end{document}